\documentclass[runningheads]{llncs}

\usepackage{eccv}

\usepackage{eccvabbrv}

\usepackage{graphicx}
\usepackage{booktabs}
\usepackage{algorithm}
\usepackage{algorithmic}
\usepackage{wrapfig}
\usepackage{enumitem}

\usepackage[accsupp]{axessibility}  

\usepackage{hyperref}

\usepackage{orcidlink}

\begin{document}

\title{Scalar Function Topology Divergence: Comparing Topology of 3D Objects} 

\titlerunning{Scalar Function Topology Divergence}

\author{Ilya Trofimov\inst{1}\orcidlink{0000-0002-2961-7368} \and
Daria Voronkova \inst{1,2} \and
Eduard Tulchinskii\inst{1,4} \and
Evgeny Burnaev\inst{1,2} \and
Serguei Barannikov\inst{1,3}}

\authorrunning{I.~Trofimov et al.}

\institute{
Skolkovo Institute of Science and Technology, Moscow, Russia
\email{\{ilya.trofimov,darya.voronkova,eduard.tulchinskiy,e.burnaev, s.barannikov\}@skoltech.ru}
\and
AIRI, Moscow, Russia\\
\and 
CNRS, IMJ, Paris Cité University, France
\and
AI Foundation and Algorithm Lab, Moscow, Russia
}

\maketitle

\begin{abstract}
We propose a new topological tool for computer vision - Scalar Function Topology Divergence (SFTD), which measures the dissimilarity of multi-scale topology between sublevel sets of two functions having a common domain. Functions can be defined on an undirected graph or Euclidean space of any dimensionality.
Most of the existing methods for comparing topology are based on Wasserstein distance between persistence barcodes and they don't take into account the localization of topological features.
The minimization of SFTD ensures that the corresponding topological features of scalar functions are located in the same places.
The proposed tool provides useful visualizations depicting areas where functions have topological dissimilarities.
We provide applications of the proposed method to 3D computer vision. In particular, experiments demonstrate that SFTD as an additional loss improves the reconstruction of cellular 3D shapes from 2D fluorescence microscopy images, and helps to identify topological errors in 3D segmentation. Additionally, we show that SFTD outperforms Betti matching loss in 2D segmentation problems. The code is publicly available:
\url{https://github.com/IlyaTrofimov/SFTD}.
\end{abstract}

\section{Introduction}


A common situation in computer vision is a structural prediction defined over~$\mathbb{R}^n$. Examples include image segmentation \cite{ronneberger2015u, hu2019topology, hu2021topology, stucki2023topologically} and 3D shape reconstruction 
\cite{wu2017marrnet, waibel2022capturing, waibel2022shapr} problems. 
However, in most of the cases the comparison between the prediction and a ground truth is done pixel(voxel)-wise by metrics like Dice score, cross-entropy, etc. 
These metrics are not capable of providing detailed topological correctness, i.e. preserving multi-scale shape patterns: cavities, holes, connectivity patterns, etc. 
Typically, predictions involve scalar functions \mbox{$f: \mathbb{R}^n \to \mathbb{R}$} and the final shape of an object is defined by thresholding $f(x)$.

There is a need for topological-aware comparison metrics for scalar functions, which take into account the topology of sublevel sets $f^{-1}((-\infty, \varepsilon])$ (the analysis for superlevel sets is similar). 
Existing tools for this problem don't take into account the localization of topological features \cite{hu2019topology, waibel2022capturing} or are limited to 2D space \cite{stucki2023topologically}. 
Our contributions are as follows:

\begin{enumerate}
\item We introduce new topological tools: F-Cross-Barcodes and Scalar Function Topology Divergence (SFTD), for a multi-scale topology-aware comparison of scalar functions by tracking differences in their sublevel sets topology.
\item A distinctive feature of F-Cross-Barcodes and SFTD is that they takes into account the localization of topological features;
\item We describe a visualization technique for F-Cross-Barcodes: a technique highlights with points areas where differences in topological features are localized;
\item We propose an algorithm for SFTD differentiation; 
\item By doing computational experiments, we show that SFTD as an additional loss term improves 3D shape reconstruction, and helps to identify topological errors in 3D segmentation.
\item We also show that SFTD outperforms Betti matching loss in terms of topological accuracy in 2D segmentation while being $\approx35$ times faster, see Appendix \ref{app:2d_segmentation}.
\end{enumerate}

\section{Related work}

Tools of topological data analysis find application in various machine learning problems: evaluation of deep generative models \cite{barannikov2021manifold}, comparison of neural representations \cite{barannikov2021representation}, dimensionality reduction \cite{moor2020topological,trofimov2023learning}, 
graph neural networks \cite{horn2021topological, yan2021link, carriere2020perslay}, self-supervised learning \cite{luo2023improving}, molecular data analysis \cite{demir2022todd}, analysis of loss functions in deep learning \cite{barannikov2023barcodes, barannikov2020topological}, spatial-temporal graph neural network tasks \cite{chen2022tamp, chen2021z}, crowd localization \cite{abousamra2021localization}, to name a few. 


Several works used topological tools to enhance image segmentation \cite{hu2019topology, hu2021topology, stucki2023topologically, gupta2024topology}. 
However, proposed algorithms for comparison ground truth and predicted images either don't take into account localization of topological features \cite{hu2019topology}, or are limited to 2D \cite{stucki2023topologically}.
A method by \cite{hu2021topology} highlights in images topologically critical and error-prone structures using discrete Morse theory, and \cite{gupta2024topology} proposes a method for inferring structure-aware uncertainty. 
The importance of topological interactions for multi-class medical image segmentation is shown in \cite{gupta2022learning}.








\section{Preliminaries}

\begin{wrapfigure}{r}{0.35\textwidth}
\vskip-0.7in
\centering
\includegraphics[width=0.35\textwidth]{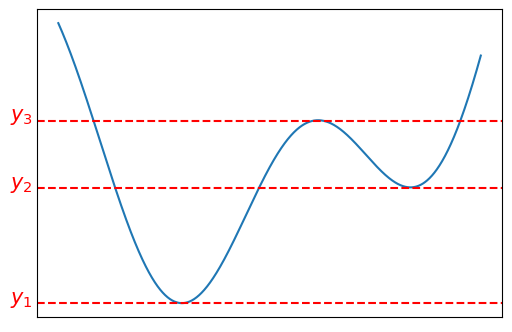}
  \caption{1D function and its sublevel sets.}
\label{fig:1d_function}
\vskip-.8in
\end{wrapfigure}

We study scalar functions $f: \mathbb{R}^n \to \mathbb{R}$ by assessing the topology of its sublevel sets. Consider a 1D function depicted in Figure \ref{fig:1d_function}. It has two local minima $y_1, y_2$ and a local maximum $y_3$, such that $y_1 < y_2 < y_3$. One can study sublevel sets $X_\varepsilon = f^{-1}((-\infty, \varepsilon])$. For the scalar function $f$, changes of $X_\varepsilon$ for increasing $\varepsilon$ can be intuitively understood as a ``filling the pool'', see Figure \ref{fig:1d_function}.
There are four distinct situations: 
\begin{enumerate}
\item For $\varepsilon < y_1$, $X_\varepsilon = \varnothing$;
\item For $y_1 \le \varepsilon < y_2$, $X_\varepsilon$ is a single segment;
\item For $y_2 \le \varepsilon < y_3$, $S_\varepsilon$ consists of 2 segments;
\item Finally, for $\varepsilon \ge y_3$, $X_\varepsilon$ is a single segment.
\end{enumerate}

Note, that changes in topology of $X_\varepsilon$ happen when $\varepsilon$ is a critical point of $f$. In the general case, smooth scalar functions defined on smooth manifolds are studied by Morse theory \cite{barannikov2023barcodes}. 
In practice, we assess functions through their values on a finite set of points, e.g. a lattice (square grid graph) in $\mathbb{R}^n$. 


\subsection{Scalar functions defined on a graph}

In the most general form, we are interested in topology-aware comparison of scalar functions $f: \mathcal{V} \to \mathbb{R}$ defined on vertices $\mathcal{V}$ of an undirected graph $\mathcal{G} = (\mathcal{V}, \mathcal{E})$. 
Specifically, we study topological features of subgraphs \mbox{$\{v \in \mathcal{V} \; | \; f(v) \le \varepsilon\}$}: connected components, cycles, etc. 

For each clique $C$ in the graph $\mathcal{G}$, we define a \textit{filtration function}
$T(C) = \max_{i \in C} f(i)$.
By a \textit{clique complex} $\mathcal{C}$, we mean a set of cliques where each sub-clique also belongs to a complex\footnote{Every clique complex is an abstract simplicial complex, but the reserve is not true.}.
For some $\varepsilon$, $T(C)$ defines a clique complex $\mathcal{C}_{j} =\{C \; | \; T(C) \le \varepsilon\}$, where $f(j) = \varepsilon$.
Topological features for a subgraph \mbox{$\{v \in \mathcal{V} \; | \; f(v) \le \varepsilon\}$} are precisely expressed by \textit{homology groups}.
Overall, the filtration function $T(C)$ defines the \textit{lower-star} filtration of the graph's clique complex:
$$
\varnothing \subseteq \mathcal{C}_0 \subseteq \mathcal{C}_1 \subseteq \ldots \subseteq \mathcal{C}_n,
$$
where $\mathcal{C}_n$ is a clique complex of $\mathcal{G}$. The filtration is said to be induced by the vertex function $f$.


A given topological feature ``appears'' at some $\varepsilon_1$ and ``dies'' at some $\varepsilon_2$. These topological features, together with their birth and death times can be calculated via tools of topological data analysis.
The multi-set of the corresponding intervals $(\varepsilon_1, \varepsilon_2)$ is called the \textit{persistence barcode}, which depicts topological features of the graph for all scales. The whole theory is dubbed \textit{persistent homology} 
\cite{barannikov1994framed,chazal2017introduction}. 


\section{Method}
\label{sec:method}

\subsection{Scalar functions defined on a graph}\label{sec:graphs}

\begin{figure}[tp]
\centering
\begin{subfigure}{.49\linewidth}
  \centering
  \includegraphics[width=0.45\linewidth]{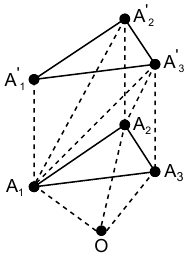}
  \caption{The doubled graph $\tilde{\mathcal{G}}$.}
  \label{fig:doubled_graph}
\end{subfigure}
\hfill
\begin{subfigure}{.49\linewidth}
  \centering
  \includegraphics[width=0.58\linewidth]{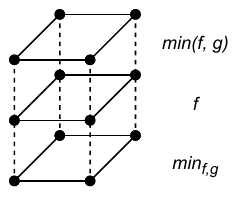}
  \caption{The extended $n+1$ dimensional lattice $\tilde{\mathcal{L}}$.}
  \label{fig:3lattice}
\end{subfigure}
\caption{A visualization of $\tilde{\mathcal{G}}$, $\tilde{\mathcal{L}}$.}
\end{figure}


To compare topological features of sublevel sets of functions $f, g: \mathcal{G} \to \mathbb{R}$ we construct a graph $\tilde{\mathcal{G}} = (\tilde{\mathcal{V}}, \tilde{\mathcal{E}})$ with doubled set of vertices from $\mathcal{G}$ and one extra vertex $O$, see Figure \ref{fig:doubled_graph}.
Formally, $\tilde{\mathcal{V}}=\{A_1, \ldots, A_n, A'_1, \ldots, A'_n, O\}$. 
Edges $\tilde{\mathcal{E}}$ are defined as follows:
\begin{gather*}
(A_i, A_j) \in \mathcal{E} \Rightarrow (A_i, A_j) \in \tilde{\mathcal{E}}, (A'_i, A'_j) \in \tilde{\mathcal{E}}  \\
(A_i, A_j) \in \mathcal{E},\; i < j \Rightarrow (A_i, A'_j) \in \tilde{\mathcal{E}} \\
(O, A_i) \in \tilde{\mathcal{E}}, (A_i, A'_i) \in \tilde{\mathcal{E}} \; \text{for all} \; i
\end{gather*}
The vertex function $\tilde{f}$ is defined as follows: 
\begin{gather*}
\tilde{f}(A_i) = f(i) \\
\tilde{f}(A'_i)= \min(f(i), g(i)), \tilde{f}(O) = {\min}_{f,g},
\end{gather*}
where $\min_{f, g}$ is the minimum value of both $f, g$ in the nodes of the graph $\mathcal{G}$. 

\textbf{Definition.} The \textit{F-Cross-Barcode}$_k(f, g)$ for two scalar function $f, g$ defined on vertices of an undirected graph $\mathcal{G}$ is the $k$-th persistence barcode for the doubled graph $\tilde{\mathcal{G}}$ and the lower-star filtration induced by $\tilde{f}$. 


To put it simply, each interval in the \textit{F-Cross-Barcode}$_k(f, g)$ denotes a range of functions’ sub-level sets where a given topological feature exists in the function $f$ but doesn’t exist in the function $g$, see further Appendix \ref{sec:stability}.
Technically, the desired persistence barcode is calculated via  \texttt{giotto-ph} library \cite{burella2021giottoph}. For this software, values of a filtration function must be defined on vertices and edges.
Algorithm \ref{alg:sfd_barcodes} summarizes the aforementioned.

\begin{algorithm}[t]
\caption{F-Cross-Barcode$_k$ (scalar function $k-$th cross-barcode)}
\label{alg:sfd_barcodes}
   
\begin{algorithmic}
\STATE \hskip-1em {\bfseries Input:} Scalar functions $f,g$ defined on vertices of a graph $\mathcal{G} = (\mathcal{V}, \mathcal{E})$. $|\mathcal{V}| = n$.
\STATE \hskip-1em {\bfseries Require:} \texttt{PB}$_k(\cdot)$ - an algorithm to calculate $k$-th persistence barcode from a matrix containing values of a filtration function on a graph: on vertices - on a diagonal, on edges - off a diagonal.
\vskip.05in


\STATE {1.} Create a $n\times n$ matrix $F$ for the function $f$:

\quad $F_{ij} = f(i)$ for $i = j$, $\max(f(i), f(j))$ for $i \neq j$ and $(i, j) \in \mathcal{E}$, $+\infty$  otherwise.
\STATE {2.} Initialize $c_{+\infty}$ to a vector of $n$ elements, all equal to $+\infty$;
\STATE {3.} Create $F_{+}$, a matrix derived from $F$ with a lower-diagonal part replaced by $+\infty$;
\STATE {4.} Create $F_{diag}$, a vector containing the diagonal elements of $F$;
\STATE {5.} Create $\min(F,G)$, a $n\times n$ matrix:
\; $\min(F,G)_{ij} = \min(f(i), g(i))$ for $i = j$,

\quad $\max(\min(f(i), g(i)), \min(f(j), g(j)))$ for $i \neq j$ and $(i, j) \in \mathcal{E}$, $+\infty$  otherwise.
\STATE {6.} Compute $\min_{F,G}$, the minimum value of all elements in $F$ and $G$;
\STATE {7.} Compute the $(2n+1)\times(2n+1)$ block matrix $m$ as follows:
$$
m\gets \begin{pmatrix}
      \min(F,G)              & F_+^\intercal      & c_{+\infty}^\intercal \\\
      F_+                    & F                  & F_{diag}^\intercal \\\ 
      c_{+\infty}  & F_{diag} & \min_{F, G}
         \end{pmatrix}
$$
\STATE {8.} Compute: F-Cross-Barcode$_k(f,g) \gets$ \texttt{PB}$_k(m)$
\vskip.05in
\STATE \hskip-1em{\bfseries Return:} F-Cross-Barcode$_k(f,g)$, a multi-set of (``birth'', ``death'') intervals.
\end{algorithmic}

\end{algorithm}

\subsection{Scalar functions defined on a lattice}
\label{sec:cubical_complex}


A frequent situation is a function having a domain in $\mathbb{R}^n$.
We assume that the function is approximated by its restriction on a sufficiently small grid. 
Algorithm~\ref{alg:sfd_barcodes} can also be applied in this case after a space triangulation. However, the number of simplices grows very quickly in high dimensions (see discussion in \cite{wagner2011efficient}), and Algorithm \ref{alg:sfd_barcodes} has a high computational cost. That is why, we consider \textit{cubical complexes} \cite{kaczynski2004computational} instead of simplicial complexes. Specifically, we consider a regular square grid graph ($n$-dimensional lattice) and use values of functions in nodes of this grid: $f(j_1, \ldots, j_n), g(j_1, \ldots, j_n)$.
To compare two function $f,g$, we create an extended $n+1$ dimensional lattice $\tilde{\mathcal{L}}$ with one extra dimension of size $3$, see Figure \ref{fig:3lattice}. The vertex function $\tilde{f}: \tilde{\mathcal{L}} \to \mathbb{R}$ is defined as follows:
\begin{gather*}
\tilde{f}(i,j_1, \ldots, j_n) =   
\begin{cases}
  \min_{f, g} & \!\! \text{for } i = 0\\    
  f(j_1, \ldots, j_n) & \!\! \text{for } i = 1 \\
  \min(f(j_1, \ldots, j_n), g(j_1, \ldots, j_n)) & \!\! \text{for } i = 2
\end{cases}
\end{gather*}

Similarly to functions on graphs, \textit{F-Cross-Barcode}$_k(f, g)$ for functions on lattices is defined as the $k-$th persistence barcode of the filtered cubical complex with the filtration induced by the vertex function $\tilde{f}$. 
We use the GUDHI\footnote{\url{https://gudhi.inria.fr/}} library which implements efficient algorithms for cubical complexes \cite{wagner2011efficient}.
Note that \mbox{\textit{F-Cross-Barcode}$_k(f, g)$} for functions restricted on a lattice and a triangulation of the same space are not equivalent, but would be similar.


\subsection{Scalar Function Topology Divergence}

The \textit{F-Cross-Barcode}$_k(f, g)$ depicts topological differences of functions' sublevel sets at multiple scales. The \textit{F-Cross-Barcode}$_k(f, g)$, as a regular barcode, is a multi-set of intervals. One requires some vectorization technique to inject it into a machine learning pipeline. If one needs to describe a discrepancy between $f, g$ by an integral numerical characteristic, we propose a \textit{scalar function topology divergence}:
\begin{equation}
\label{eq:sftd}
\mbox{SFTD}_k(f, g) = \sum_{(b, d) \in \text{F-Cross-Barcode}_k(f, g)}{||d - b||^p},
\end{equation}
that is the sum of lengths of intervals from the \textit{F-Cross-Barcode}$_k(f, g)$ to $p$-th power, where typically $p=1,2$.
SFTD is not symmetric, and we often use an average of $\mbox{SFTD}_k(f, g)$ and $\mbox{SFTD}_k(g, f)$ in experiments below.
For different dimensions $k$, $\mbox{SFTD}_k(f, g)$ can be used either separately, or the sum can be taken to reflect discrepancies in all types of topological features (connected components, cycles, voids, etc.). Computational complexity of SFTD is discussed in Appendix \ref{app:complexity}. 

\subsection{SFTD basic properties}
\begin{proposition} 
\label{prop_std}
If $SFT\!D_k(f, g) = SFT\!D_k(g, f)=0$ for all $k \geq 0$, then the sublevel sets barcodes of functions  $f$ and $g$ are the same in any degree. Moreover, in this case their topological features are located in the same places: 
the inclusions $ f^{-1}((-\infty, \varepsilon])\subseteq  (\min(f,g))^{-1}((-\infty, \varepsilon])$,  $ g^{-1}((-\infty, \varepsilon])\subseteq  (\min(f,g))^{-1}((-\infty, \varepsilon])$, induce homology isomorphisms for any threshold~$\varepsilon$.
\end{proposition}

\begin{proposition}\label{prop:stab}\textbf{(SFTD Stability)}

  (a) For any scalar functions $f$, $f'$, $g$, 
 $g'$ on a grid, or on a graph vertex set: 
  \vskip-0.25in
\begin{multline}
    d_B(\text{F-Cross-Barcode} _ k(f,g), \text{F-Cross-Barcode} _ k(f',g')) \leq \\ \max(\max _ {i} \lvert f(i)-f'(i)\rvert ,\max _ {i} \lvert g(i)-g'(i)\rvert ). \nonumber
\end{multline}
 \vskip-0.15in
(b) For any  pair of scalar functions $f$, $f'$: 
 \vskip-0.1in
 $$\lVert \text{F-Cross-Barcode}_k(f,f')\rVert_B \leq \max_{i} \lvert f(i)-f'(i)\rvert,$$ 
\vskip-0.1in
\end{proposition}

where $d_B$ denotes the bottleneck distance between persistence barcodes \cite{chazal2017introduction}.
We give the proofs in Appendix~\ref{sec:stability}.
As a consequence, SFTD is robust to noise.

\subsection{Visualization of F-Cross-Barcodes}

\textit{F-Cross-Barcodes} can be used for localization and visualization of topological differences between scalar functions. Recall the filtered cubical\footnote{For ease of presentation, we discuss only the visualization for cubical complexes. The visualization for clique complexes is analogous.} complex:
$$
\varnothing \subseteq \mathcal{C}_0 \subseteq \mathcal{C}_1 \subseteq \ldots \subseteq \mathcal{C}_n.
$$
Since the filtration is induced by a vertex function $\tilde{f} : \tilde{\mathcal{L}} \to \mathbb{R}$, the filtered cubical complex corresponds to the sequence:
$$
\tilde{f}(i_0) \le \tilde{f}(i_1) \le \ldots \le \tilde{f}(i_n).
$$
If $\tilde{f}(i_k) = \tilde{f}(i_m)$, the order is defined by indexes $i_k, i_m$.
Thus, each segment in the barcode - a birth/death pair - is associated with two vertices $v_b, v_d \in \tilde{\mathcal{L}}$. 
Recall that $\tilde{\mathcal{L}}$ is an extended $n+1$ dimensional lattice (see Section \ref{sec:cubical_complex}), where the first dimension is of size 3, and the rest of dimensions are of the same size as the lattice inside the domain of $f, g$. 
So, neglecting the first dimension, we can map the rest of indices $j_1, \ldots, j_n$ of birth/death vertices $v_b, v_d$ right onto original $n-$dimensional lattice.
Later, in the experimental section, we shall provide such visualizations, where birth/death events are depicted by orange/red dots.


\subsection{SFTD differentiation}

Calculation of \mbox{SFTD}$_k(f, g)$ involves a filtered simplicial(cubical) complex $\mathcal{C}$ with a vertex function $\tilde{f}(i), i \in \tilde{\mathcal{V}}$. 
Also, since we apply the sub-level set filtration, for every simplex(cube) $\sigma$: $T(\sigma) = \tilde{f}(i_\sigma)$, where $i_\sigma = \text{argmax}_{i \in \sigma} \tilde{f}(i)$.
Each interval $(b,d)$ in a \textit{F-Cross-Barcode}$_k$ corresponds to some pair of simplices(cubes) $(\sigma_1, \sigma_2)$ which join the simplicial(cubical) complex $\mathcal{C}$ at moments $b=T(\sigma_1), d=T(\sigma_2)$.
Generally, we consider $f,g$ (and consequently $\tilde{f}$) to be parametric functions of $w$, for example, neural networks.
By differentiating (\ref{eq:sftd}) we obtain:
\begin{gather}
\frac{\partial \, \mbox{SFTD}_k(f, g)}{\partial w} 
= \sum_{(\sigma_1, \sigma_2) \in \Sigma_k} p \left( \frac{\partial \tilde{f}(i_{\sigma_2})}{\partial w} - \frac{\partial \tilde{f}(i_{\sigma_2})}{\partial w} \right)^{p-1},
\label{grad1}
\end{gather}
where $\Sigma_k$ is a collection of ``birth'',``death'' simplex pairs corresponding to intervals in the \textit{F-Cross-Barcode}$_{k}(f, g)$.


By construction, $\tilde{f}(i)$ can be either $f(i')$ or $g(i')$ for some mapping $i'=r(i)$. Formally, introducing a binary function $c(\cdot)$ for this choice:
\begin{equation}
\tilde{f}(i) = c(i)f(r(i)) + (1 - c(i))g(r(i)),
\end{equation}
and 
\begin{equation}
\frac{\partial \tilde{f}(i)}{\partial w} = b(i)\frac{\partial f(r(i))}{\partial w} + (1-b(i))\frac{\partial g(r(i))}{\partial w}.
\label{grad2}
\end{equation}

Finally, the desirable gradient $\partial \,\mbox{SFTD}_k (f, g) / \partial w$ can be calculated by combining (\ref{grad1}), (\ref{grad2}) via a chain rule.
With a common assumption that all the $f(i), g(i)$ are distinct, functions $c, r$ don't change in a vicinity of $w$, and $\mbox{SFTD}_i(f, g)$ is differentiable w.r.t. $w$ \cite{hofer2020graph, carriere2021optimizing}.
On a high level, the differentiation is done similarly to
\cite{carriere2021optimizing, trofimov2023learning}.

\section{Experiments}

In the experimental section, we are going to answer three research questions:
\begin{itemize}
\item[\textbf{Q1:}]  Does SFTD and F-Cross-Barcodes provide a better topologically-aware comparison of scalar functions than Wasserstein distance between regular persistence barcodes?
\item[\textbf{Q2:}] Does F-Cross-Barcodes help to visualize and localize topological differences?
\item[\textbf{Q3:}] Is SFTD, as a loss function, beneficial for computer vision problems? 
\end{itemize}

\subsection{Synthetic data}

To illustrate the performance of the proposed tools, we start with synthetic data. 
Functions under comparison have coinciding persistence barcodes of all degrees and they are indistinguishable by standard Wasserstein matching.

\subsubsection{Functions with local minima in 2D.}

\begin{figure}[tp]
\centering
\begin{subfigure}{.32\linewidth}
  \centering
  \includegraphics[width=1\linewidth]{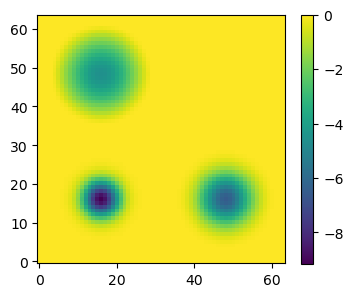}
  \caption{Scalar func. $f_1$.}
\end{subfigure}%
\hfill
\begin{subfigure}{.32\linewidth}
  \centering
  \includegraphics[width=1\linewidth]{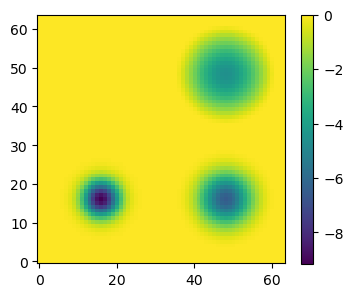}
  \caption{Scalar func. $f_2$.}
\end{subfigure}
\hfill
\begin{subfigure}{.275\linewidth}
  \centering
  \includegraphics[width=1\linewidth]{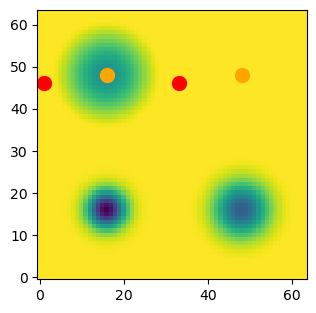}
  \caption{Vis. diff.}
\end{subfigure}

\begin{subfigure}{.32\linewidth}
  \centering
  \includegraphics[width=1\linewidth]{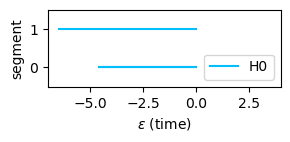}
  \caption{Pers. barc. of $f_1$.}
\end{subfigure}%
\hfill
\begin{subfigure}{.32\linewidth}
  \centering
  \includegraphics[width=1\linewidth]{pic/3_minima/barcode.png}
  \caption{Pers. barc. of $f_2$.}
\end{subfigure}
\hfill
\begin{subfigure}{.32\linewidth}
  \centering
  \includegraphics[width=1\linewidth]{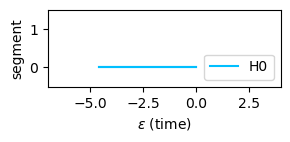}
  \caption{F-Cross-Barcode}
\end{subfigure}

\caption{Two scalar functions (a), (b) are indistinguishable by their persistence barcodes (d), (e). At the same time, $\text{SFTD}_0(f_1, f_2)=4.6$ see (f).
}
\label{fig:3minima}
\end{figure}

First, we compare two scalar function on  64$\times$64 grid, each of them having 3 local minima but in different locations, see Figure \ref{fig:3minima}, (a), (b). Persistence barcodes of both functions are equal, see Figure \ref{fig:3minima}, (d), (e): they contain two segments from $H_0$ corresponding to local minima. 
Here and further, essential segments (which die at $+\infty$) are omitted.
The topology of the two functions  is indistinguishable by standard persistence barcodes and their Wasserstein distance is zero. At the same time, \textit{F-Cross-Barcode} between these scalar functions is non-empty for $H_0$, see Figure \ref{fig:3minima}, (f). The visualization technique correctly points at places with different topology - the centers and borders of the distinguishing local minima.

\subsubsection{Lattices in 2D.}

\begin{figure}[tp]
\centering
\begin{subfigure}{.32\linewidth}
  \centering
  \includegraphics[width=1\linewidth]{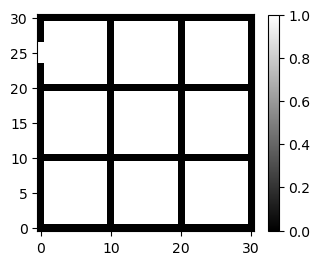}
  \caption{Scalar func. $f_1$.}
\end{subfigure}%
\hfill
\begin{subfigure}{.32\linewidth}
  \centering
  \includegraphics[width=1\linewidth]{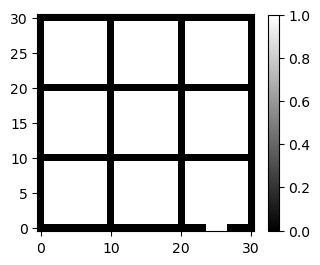}
  \caption{Scalar func. $f_2$.}
\end{subfigure}
\hfill
\begin{subfigure}{.27\linewidth}
  \centering
  \includegraphics[width=1\linewidth]{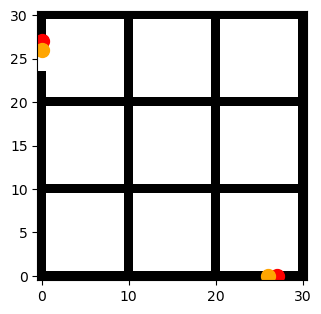}
  \caption{Vis. diff.}
\end{subfigure}

\begin{subfigure}{.32\linewidth}
  \centering
  \includegraphics[width=1\linewidth]{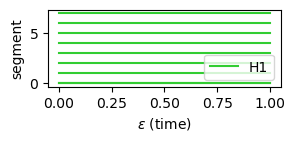}
  \caption{Pers. barc. of $f_1$.}
\end{subfigure}%
\hfill
\begin{subfigure}{.32\linewidth}
  \centering
  \includegraphics[width=1\linewidth]{pic/2d_lattice/barcode.png}
  \caption{Pers. barc. of $f_2$.}
\end{subfigure}
\begin{subfigure}{.32\linewidth}
  \centering
  \includegraphics[width=1\linewidth]{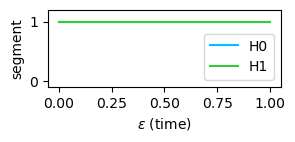}
  \caption{F-Cross-Barcode.}
\end{subfigure}
\caption{
Two scalar functions (a), (b) are indistinguishable by their persistence barcodes (c), (d). At the same time, $\text{SFTD}_1(f_1, f_2)=1$, see (e).
}
\label{fig:2d_lattice}
\end{figure}

Next, we compare two lattice-style 2D scalar functions, see Figure \ref{fig:2d_lattice}, (a), (b). Lattices have defects in different places (top left and bottom right respectively). The persistence barcodes of both lattices are identical, see Figure \ref{fig:2d_lattice}, (d), (e): they contain 8 segments from $H_1$ corresponding to independent cycles.
At the same time, \textit{F-Cross-Barcode} between two scalar functions is non-empty, see Figure \ref{fig:2d_lattice}, (f). The visualization technique correctly points at places with defects.  

\subsubsection{Concentric spheres in 3D.}

\begin{figure}[t]
\centering
\begin{subfigure}{.32\linewidth}
  \centering
  \includegraphics[width=1\linewidth]{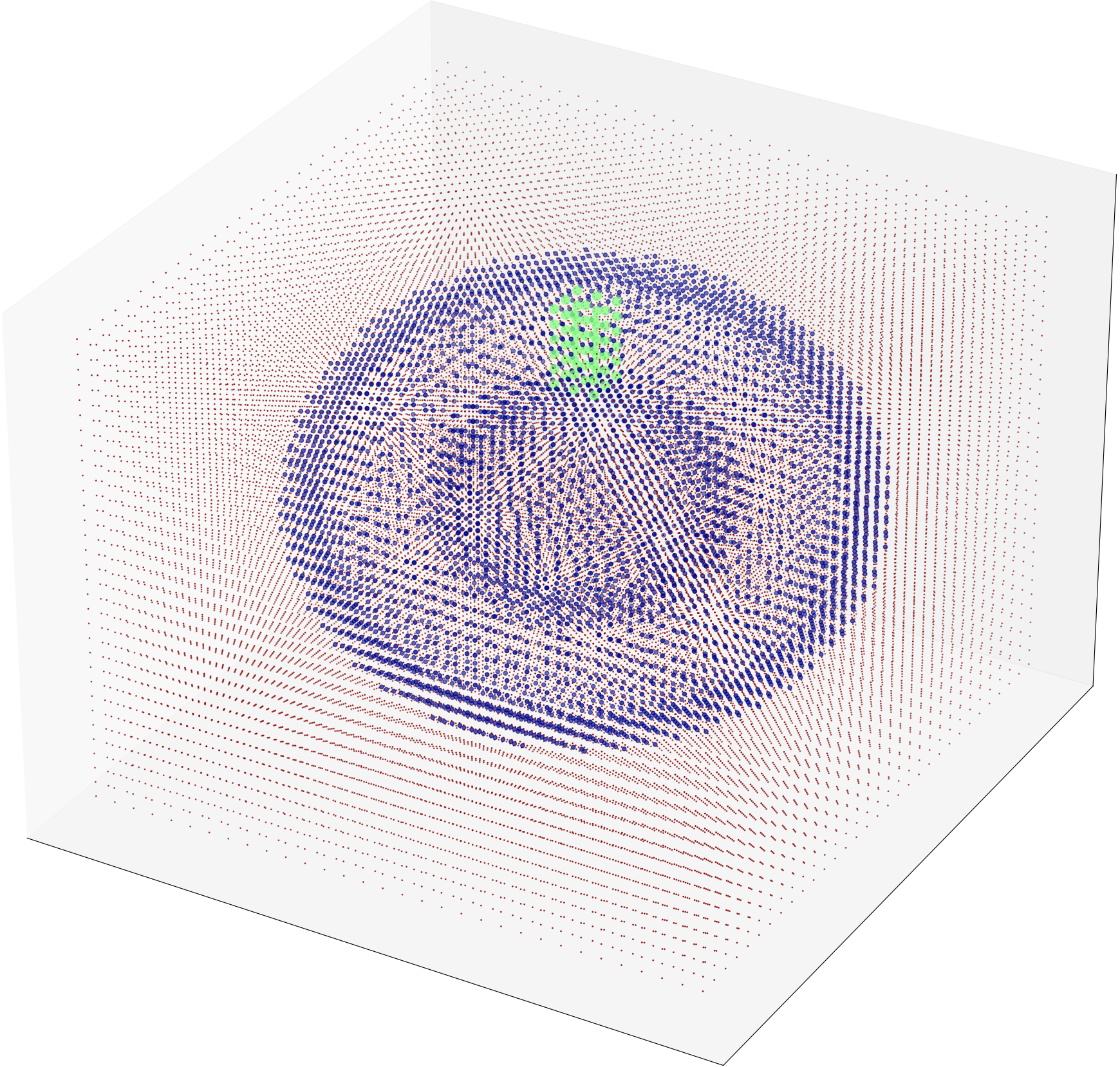}
  \caption{Scalar func. $f_1$.}
\end{subfigure}%
\hfill
\begin{subfigure}{.32\linewidth}
  \centering
  \includegraphics[width=1\linewidth]{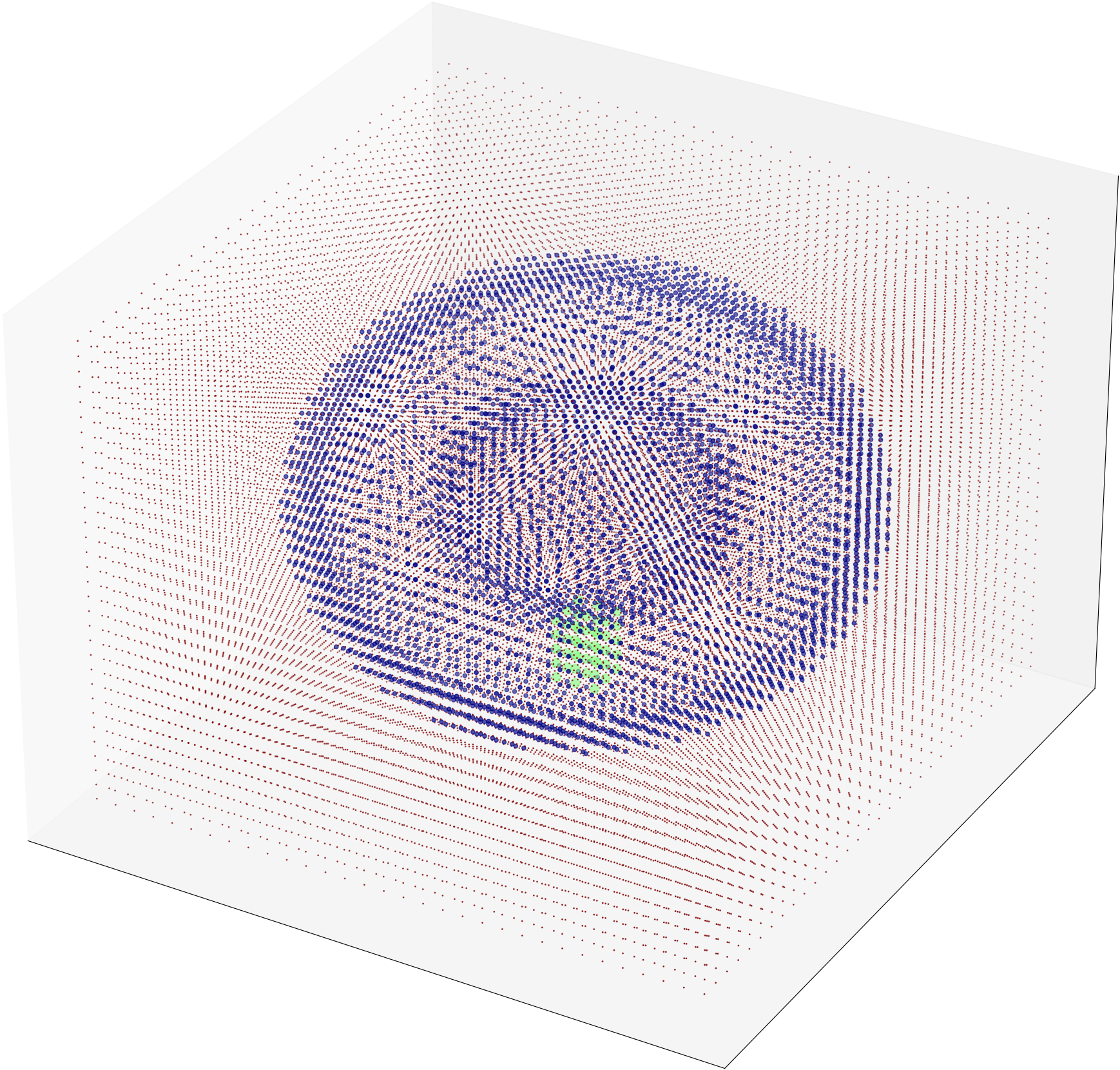}
  \caption{Scalar func. $f_2$.}
\end{subfigure}
\hfill
\begin{subfigure}{.32\linewidth}
  \centering
  \includegraphics[width=1\linewidth]{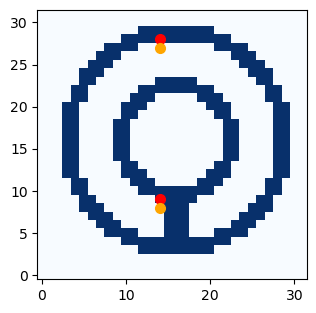}
  \caption{Vis. diff.}
\end{subfigure}

\begin{subfigure}{.32\linewidth}
  \centering
  \includegraphics[width=1\linewidth]{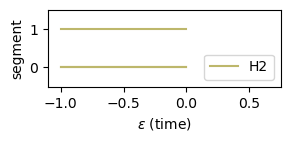}
  \caption{Pers. barc. of $f_1$.}
\end{subfigure}%
\hfill
\begin{subfigure}{.32\linewidth}
  \centering
  \includegraphics[width=1\linewidth]{pic/3d_spheres/barcode.png}
  \caption{Pers. barc. of $f_2$.}
\end{subfigure}
\begin{subfigure}{.32\linewidth}
  \centering
  \includegraphics[width=1\linewidth]{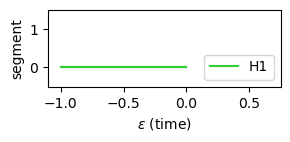}
  \caption{F-Cross-Barcode.}
\end{subfigure}
\caption{
Two scalar fields (a), (b) are indistinguishable by their persistence barcodes (d), (e). At the same time, $\text{SFTD}_1(f_1, f_2)=1$, see (f).}
\label{fig:3d_spheres}
\end{figure}

Finally, we study the performance of the proposed method by comparing 3D scalar functions.
We compare two scalar functions, which equal $-1$ in a vicinity of two concentric spheres and a bridge connecting them, which is depicted in green.
The bridge is located in different places (above or below of the smallest sphere),  
In other areas of space, the scalar functions equal $0$.
We sample a $32^{3}$ grid from the unit cube, see Figure \ref{fig:3d_spheres} (a), (b) and compare functions restricted on this grid  via \textit{F-Cross-Barcode}. Again, regular persistence barcodes of two spheres with bridges are the same, see Figure \ref{fig:3d_spheres}, (d), (e): they contain 2 segments from $H_2$ corresponding to two cavities. But \textit{F-Cross-Barcode} is non-empty, see Figure \ref{fig:3d_spheres}, (f). The visualization technique correctly highlights areas with topological differences, see Figure \ref{fig:3d_spheres} (c), where the central vertical slice is shown.

\subsection{Comparing Eigenvectors of Graph Laplacian}

\begin{figure}[tp]
\begin{minipage}{0.48\textwidth}
\centering
\includegraphics[width=\linewidth]{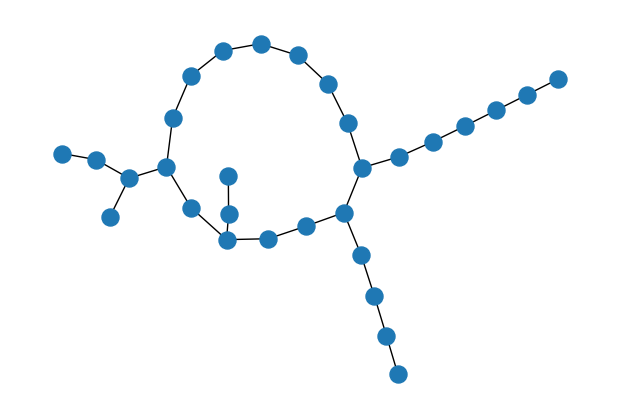}
\caption{A random Watts–Strogatz small-world graph \cite{watts1998collective} with 30 nodes.}
\label{fig:graph}
\end{minipage}
\hfill
\begin{minipage}{0.48\textwidth}
\centering
\includegraphics[width=0.97\linewidth]{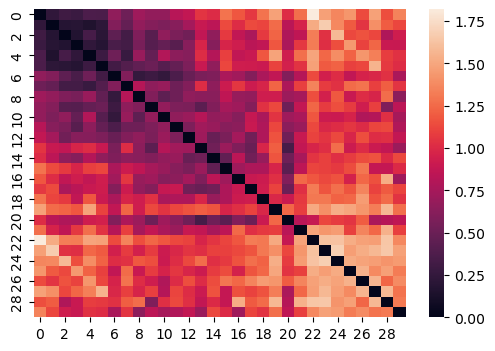}
  \caption{Comparison of graph Laplacian eigenvectors via SFTD.}
\label{fig:graph_eigenvectors_heatmap}
\end{minipage}
\end{figure}

\begin{figure*}[tp]
\centering
\begin{subfigure}{.24\linewidth}
  \centering
  \includegraphics[width=1\linewidth]{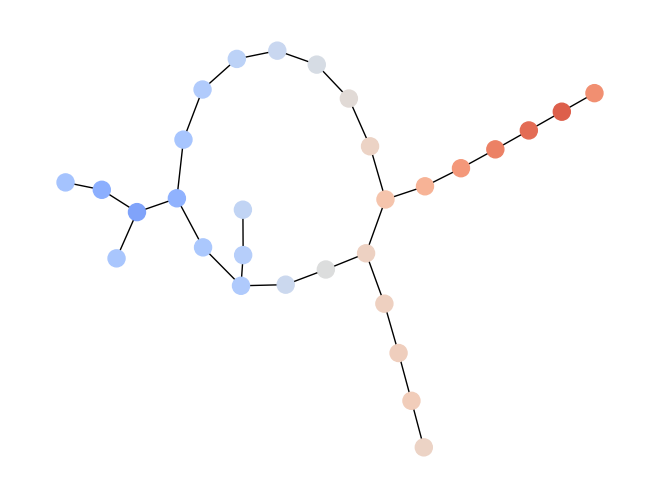}
  \caption{Eigenvector 1.}
\end{subfigure}%
\hfill
\begin{subfigure}{.24\linewidth}
  \centering
  \includegraphics[width=1\linewidth]{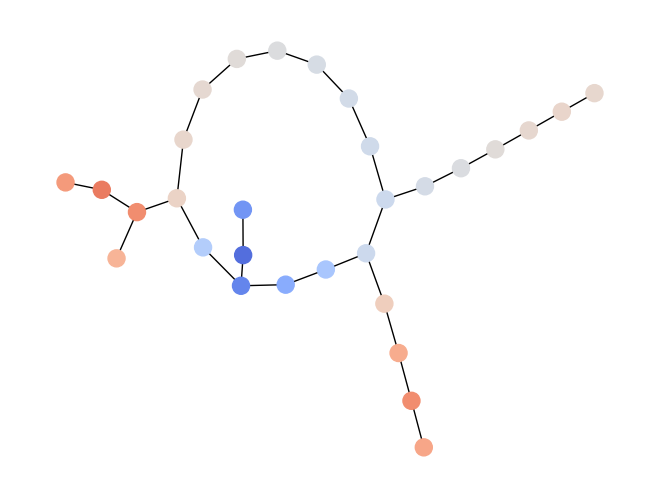}
  \caption{Eigenvector 4.}
\end{subfigure}
\begin{subfigure}{.24\linewidth}
  \centering
  \includegraphics[width=1\linewidth]{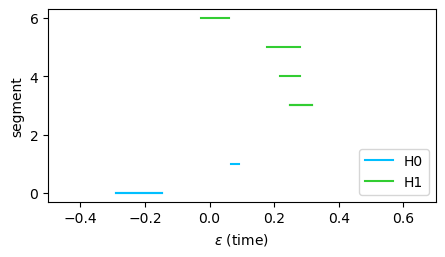}
  \caption{F-Cross-Barcode($e_1$, $e_{4}$).}
\end{subfigure}%
\hfill
\begin{subfigure}{.24\linewidth}
  \centering
  \includegraphics[width=1\linewidth]{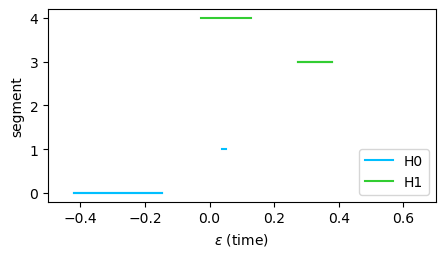}
  \caption{F-Cross-Barcode($e_{4}$, $e_1$).}
\end{subfigure}
\caption{Comparison of $e_1$ and $e_4$.}
\label{fig:graph_eigenvector_similar}
\end{figure*}

\begin{figure*}[tp]
\centering
\begin{subfigure}{.24\linewidth}
  \centering
  \includegraphics[width=1\linewidth]{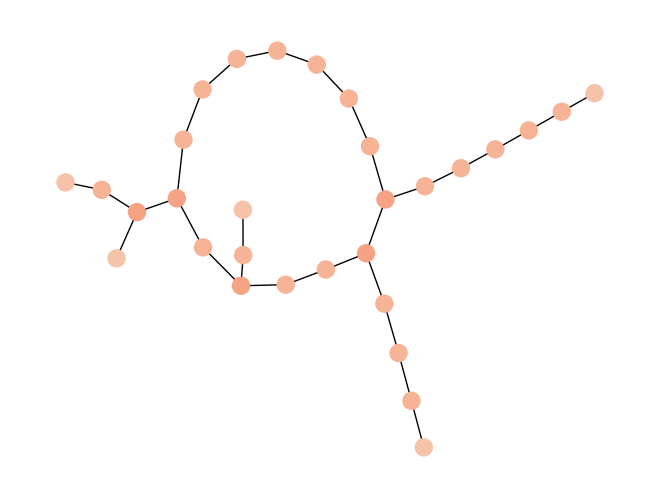}
  \caption{Eigenvector 0.}
\end{subfigure}%
\hfill
\begin{subfigure}{.24\linewidth}
  \centering
  \includegraphics[width=1\linewidth]{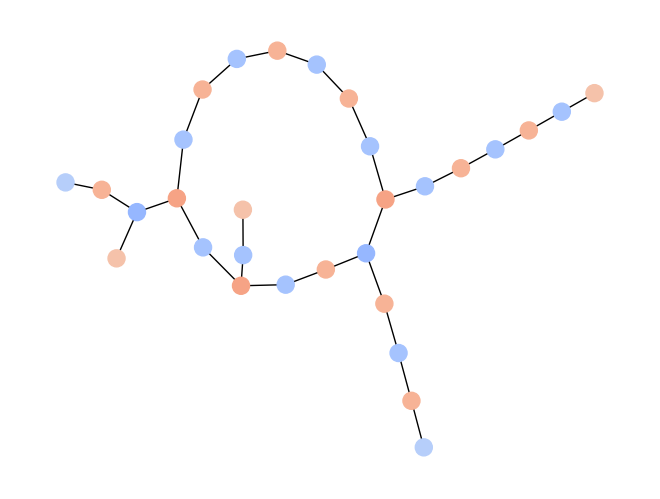}
  \caption{Eigenvector 22.}
\end{subfigure}
\begin{subfigure}{.24\linewidth}
  \centering
  \includegraphics[width=1\linewidth]{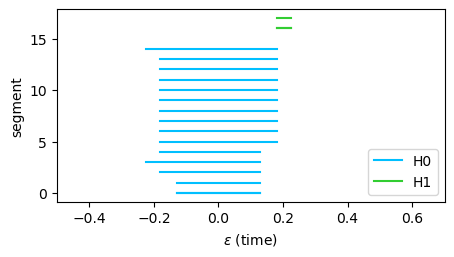}
  \caption{F-Cross-Barcode($e_0$, $e_{22}$).}
\end{subfigure}%
\hfill
\begin{subfigure}{.24\linewidth}
  \centering
  \includegraphics[width=1\linewidth]{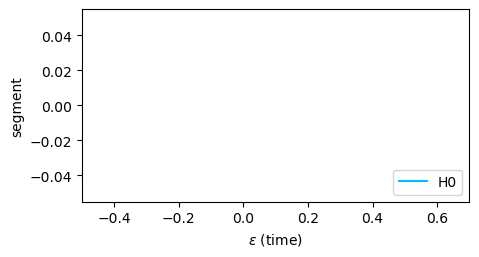}
  \caption{F-Cross-Barcode($e_{22}$, $e_0$).}
\end{subfigure}
\caption{Comparison of $e_0$ and $e_{22}$.}
\label{fig:graph_eigenvector_dissimilar}
\end{figure*}

In the most general form, the proposed tools can compare scalar functions $f: \mathcal{V} \to \mathbb{R}$ defined on vertices $\mathcal{V}$ of an undirected graph $\mathcal{G} = (\mathcal{V}, \mathcal{E})$. 
The normalized graph Laplacian is the matrix
$L = D^{-1/2}(D-A)D^{-1/2}$
where $A$ is an adjacency matrix and $D$ is a diagonal matrix of node degrees.
Eigenvectors $e_i$ and eigenvalues $\lambda_i$ of $L$ describe intrinsic properties of the graph $\mathcal{G}$. 
Eigenvectors can be considered as scalar functions defined on graph's vertices $\mathcal{V}$. For eigenvectors: $(e_i, e_j) = 0$ and their norm is not defined; one can put $||e_i||=1$ for convenience. 
Thus, $e_i$ can't be distinguished by Euclidean distance or cosine distance.
Alternatively, we can study $e_i$ as functions on vertices and compare them via SFTD.

We consider a random graph with 30 nodes, see Figure \ref{fig:graph}. Figure \ref{fig:graph_eigenvectors_heatmap} presents symmetrized $\sum_{k=0,1}\mbox{SFTD}_k(e_i, e_j)$ via a heatmap, where $e_i, e_j$ are graph Laplacian eigenvectors ordered by eigenvalues. We observed the following interesting patterns: block structure, high topological similarity in the top left block, low topological similarity in the bottom right block, and very interesting diagonal lines parallel to the main diagonal.
Figure \ref{fig:graph_eigenvector_similar} shows two the most similar eigenvectors, Figure \ref{fig:graph_eigenvector_dissimilar} - two the most dissimilar ones, while their \textit{F-Cross-Barcodes} are shown on the right.
\subsection{Enhancing 3D reconstruction}

\begin{table*}[tbp]
\caption{Errors of 3D shape reconstruction for ``Red blood cell'' dataset.}
\label{tbl:SHAPR_red_blood_cell}
\vskip 0.15in
\begin{center}
\begin{small}
\begin{sc}
\begin{tabular}{lccccccc}
\toprule
Method & 1-IoU & Volume & Surface & Rough. & $H_0$ & $H_1$ & $H_2$\\
\midrule
SHAPR \cite{waibel2022shapr}  & 0.487 & 0.353 & 0.243 & 0.361 & 0.098 & 0.241 & 0.065 \\
SHAPR+W.D. \cite{waibel2022capturing} & 0.467 & 0.291 &  0.184 & 0.294 & 0.098 & 0.228 & 0.032\\
SHAPR+$\sum_{i=0, 1, 2} \text{SFTD}_i$  & \underline{\textbf{0.449}} & 0.294 & \textbf{0.176} & 0.282 & \textbf{0.079} & 0.222 & 0.065\\
SHAPR+$\sum_{i=0, 1, 2, 3} \text{SFTD}_i$  & 0.453 & \underline{\textbf{0.266}} & 0.189 & \textbf{0.276} & \underline{\textbf{0.079}} & \underline{\textbf{0.217}} & \underline{\textbf{0.025}}\\
\bottomrule
\end{tabular}
\end{sc}
\end{small}
\end{center}
\end{table*}

\begin{table*}[hp!]
\caption{Errors of 3D shape reconstruction for ``Cell nuclei'' dataset.}
\label{tbl:SHAPR_cell_nuclei}
\begin{center}
\begin{small}
\begin{sc}
\begin{tabular}{lccccccc}
\toprule
Method & 1-IoU & Volume & Surface & Rough. & $H_0$ & $H_1$ & $H_2$\\
\midrule
SHAPR \cite{waibel2022shapr}            & 0.612 & 0.482 & 0.273 & {0.191} & \textbf{0.000} & \textbf{0.007} & 0.038 \\
SHAPR+W.D. \cite{waibel2022capturing} & 0.612 & {0.437} & 0.251 & 0.200 & \textbf{0.000} & 0.008 & 0.027\\
SHAPR+$\sum_{i=0, 1, 2} \text{SFTD}_i$  & {0.618} & \textbf{0.432} & \textbf{0.244} & \textbf{\underline{0.187}} & \textbf{0.000} & {0.008} & 0.020\\
SHAPR+$\sum_{i=0, 1, 2, 3} \text{SFTD}_i$  & \textbf{\underline{0.597}} & {0.437} & {0.247} & {0.191} & \textbf{0.000} & 0.011 & \underline{\textbf{0.011}}\\

\bottomrule
\end{tabular}
\end{sc}
\end{small}
\end{center}
\end{table*}


SHAPR (SHApe PRediction) \cite{waibel2022shapr} is a deep learning model for reconstructing cellular and nuclear 3D shapes from 2D fluorescence microscopy images. 
Correct reconstruction of structural patterns in 3D shapes is critical for this problem. However, such information is not captured by established pixel/voxel-wise loss terms, e.g. Dice loss.
Later, \cite{waibel2022capturing} improved SHAPR by adding contextual information about the overall shape of an object. 
The work \cite{waibel2022capturing} proposed to augment SHAPR with a topological loss term - a Wasserstein-2 distance (W.D.) between persistence barcodes of ground truth and predicted 3D objects and a total persistence.
However, as we discussed before, this loss term doesn't take into account the localization of topological features.
We replaced the topological loss term in SHAPR with a symmetrized SFTD with $p = 2$ in two variants\footnote{A 3D cubical complex has empty $H_3$ homologies. But during the calculation of SFTD for 3D scalar functions, a 4D cubical complex is created, where $H_3$ homologies can exist.}: dimensions up to 2 and dimensions up to 3. Then, we carried out experiments with two datasets from \cite{waibel2022shapr}. 

Table~\ref{tbl:SHAPR_red_blood_cell} shows experimental results for the ``Red blood cell'' dataset.
We conclude that SHAPR augmented with the SFTD loss term is the best one in terms of 3D shape matching errors : IoU error, volume error, surface error, and roughness error (see details in \cite{waibel2022shapr}). Additionally, we calculated Wasserstein-2 distances between persistence barcodes ($H_0$, $H_1$, $H_2$) of ground truth and generated 3D objects. Statistically significant improvements are underlined. The proposed SFTD topological loss term with 0-3 dimensions has the lowest Wasserstein-2 distances. This is the rare case of an application of higher homologies like $H_3$ for real-world applications.

For the ``Cell nuclei'' dataset results are presented in Table \ref{tbl:SHAPR_cell_nuclei}. Differences are quite small and are statistically significant only for two metrics: IoU error and W.D. $H_2$, where SHAPR with SFTD loss is the best model.

\begin{figure}[tbp]
\centering
\begin{subfigure}{.32\linewidth}
  \centering
  \includegraphics[width=1\linewidth]{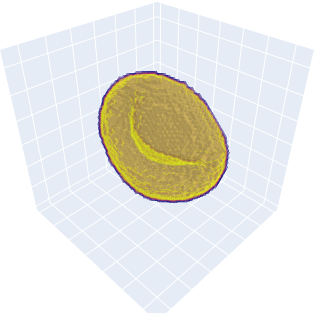}
  \caption{Ground truth}
\end{subfigure}%
\hfill
\begin{subfigure}{.32\linewidth}
  \centering
  \includegraphics[width=1\linewidth]{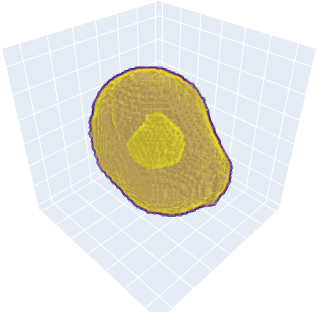}
  \caption{SHAPR+W.D.}
\end{subfigure}
\hfill
\begin{subfigure}{.32\linewidth}
  \centering
  \includegraphics[width=1\linewidth]{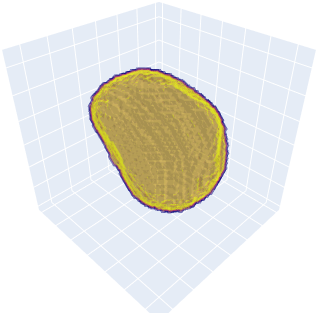}
  \caption{SHAPR+SFTD}
\end{subfigure}

\begin{subfigure}{.32\linewidth}
  \centering
  \includegraphics[width=1\linewidth]{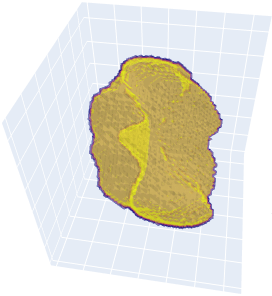}
  \caption{Ground truth}
\end{subfigure}%
\hfill
\begin{subfigure}{.32\linewidth}
  \centering
  \includegraphics[width=1\linewidth]{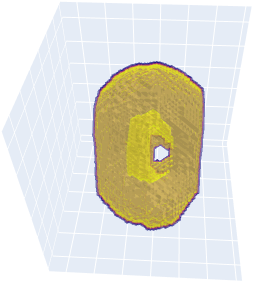}
  \caption{SHAPR+W.D.}
\end{subfigure}
\begin{subfigure}{.32\linewidth}
  \centering
  \includegraphics[width=1\linewidth]{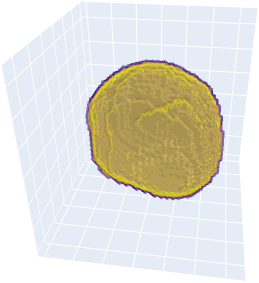}
  \caption{SHAPR+SFTD}
\end{subfigure}

\begin{subfigure}{.32\linewidth}
  \centering
  \includegraphics[width=1\linewidth]{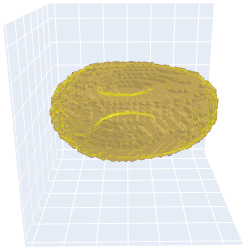}
  \caption{Ground truth}
\end{subfigure}%
\hfill
\begin{subfigure}{.32\linewidth}
  \centering
  \includegraphics[width=1\linewidth]{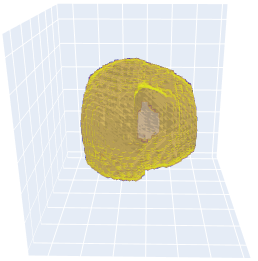}
  \caption{SHAPR+W.D.}
\end{subfigure}
\begin{subfigure}{.32\linewidth}
  \centering
  \includegraphics[width=1\linewidth]{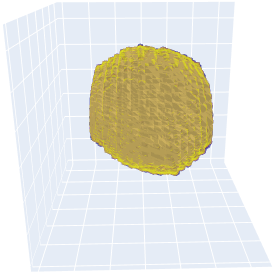}
  \caption{SHAPR+SFTD}
\end{subfigure}

\caption{Examples of topological errors. Top, $H_0$: SHAPR+W.D. has a wrong cavity inside. Middle, $H_1$: SHAPR+W.D. has a wrong hole. Bottom, $H_2$: SHAPR+W.D. has a wrong surface break.}
\label{fig:shapr_errors}
\end{figure}

Figure \ref{fig:shapr_errors} shows 3D visualizations of 2D $\to$ 3D reconstructions. We see that even with the topological regularization from \cite{waibel2022capturing}, 3D reconstructions are sometimes erroneous and have serious topological defects: cavities, holes, and surface breaks. In experiments, the sublevel filtration was used and a density function equaled ``1'' inside the object and ``0'' in the ambient space. That is, a cavity inside a sphere corresponds to an extra bar in $H_0$ barcode, an extra connected component corresponds to an extra bar in $H_2$ barcode.


\subsection{Identifying topological errors in 3D segmentation}


Brain tumor 3D segmentation aims to reveal the sub-regions of tumor corresponding to the enhancing tumor, the edematous/invaded tissue, the necrosis. Availability of versatile similarity measures not only allows precise evaluation of deep learning 3D segmentation models, but also assists monitoring disease progression in clinical practice. As mentioned above, measuring the similarity of the two 3D segmentations by comparing their persistence barcodes does not account for localization of the topological features. In this section, we illustrate the ability of SFTD to complement the existing metrics with this additional information.

In order to investigate the applicability of the proposed SFTD as a metric for 3D segmentation, we utilize the publicly available pretrained Swin UNETR model \cite{hatamizadeh2022swin, tang2022self-supervised} and BraTS 21 dataset \cite{baid2021rsna, menze2015multimodal, bakas2017advancing, bakas2017segmentation}. Figures \ref{3d_segm_label}, \ref{3d_segm_pred} demonstrate the ground truth segmentation and model’s prediction. While each of these two 3D segmentations has two voids inside, only two voids have the same localization which makes the two 3D segmentations different. As shown in Figures \ref{3d_segm_pers_barc_label}, \ref{3d_segm_pers_barc_pred}, the two 3D segmentations are indistinguishable by their persistence barcodes. In contrary, SFTD is able to recognise the difference between these two cases as revealed by the F-Cross-Barcode in Figure \ref{3d_segm_sfd_barc}. In particular, SFTD indicates that the difference is precisely in locations of one pair of voids since the set of $1$-dimensional topological features contains one segment.

\begin{figure}[tbp]
\centering
\begin{subfigure}{.32\linewidth}
  \centering
  \includegraphics[width=1\linewidth]{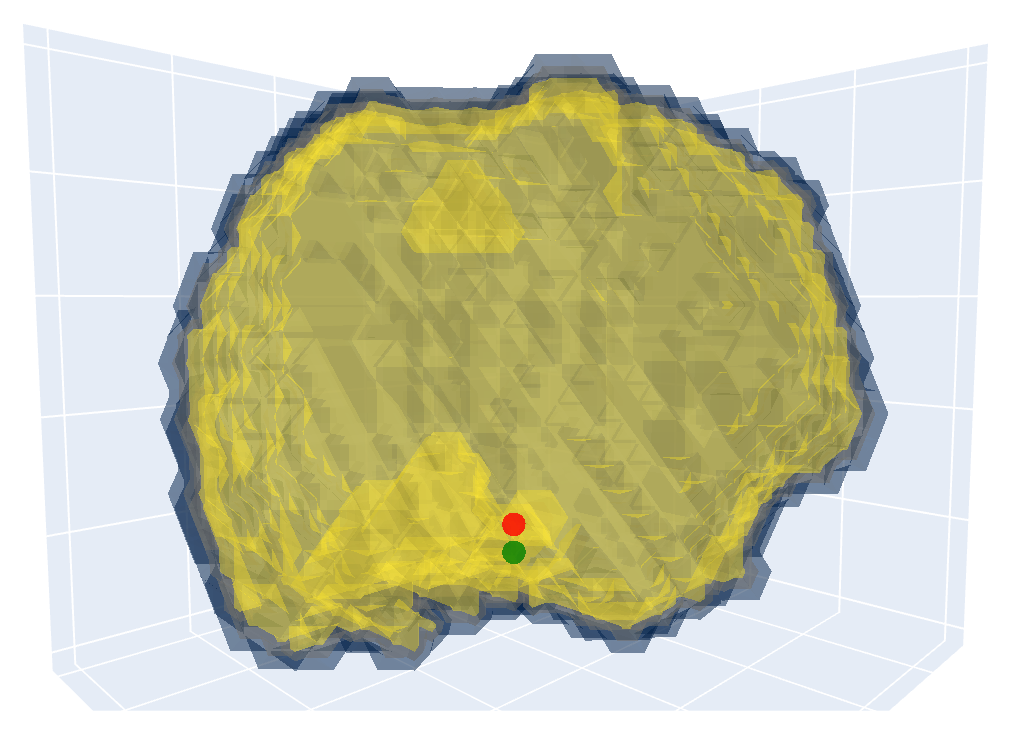}
  \caption{Ground truth}
  \label{3d_segm_label}
\end{subfigure}%
\hfill
\begin{subfigure}{.32\linewidth}
  \centering
  \includegraphics[width=1\linewidth]{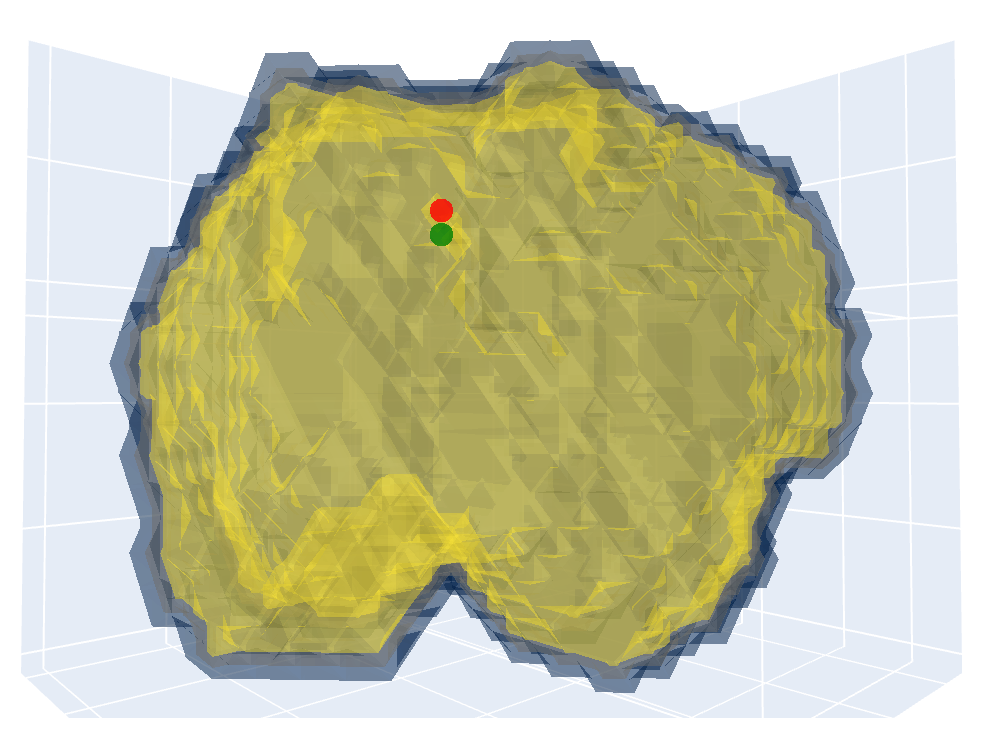}
  \caption{Prediction}
  \label{3d_segm_pred}
\end{subfigure}
\begin{subfigure}{.32\linewidth}
  \centering
  \includegraphics[width=1\linewidth]{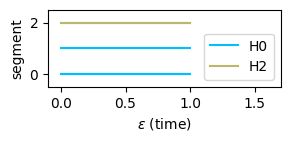}
  \caption{Pers. barc. of ground truth}
  \label{3d_segm_pers_barc_label}
\end{subfigure}%
\hfill
\begin{subfigure}{.32\linewidth}
  \centering
  \includegraphics[width=1\linewidth]{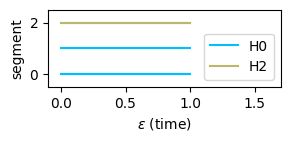}
  \caption{Pers. barc. of prediction}
  \label{3d_segm_pers_barc_pred}
\end{subfigure}
\hfill
\begin{subfigure}{.32\linewidth}
  \centering
  \includegraphics[width=1\linewidth]{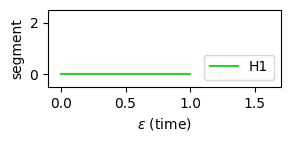}
  \caption{F-Cross-Barcode}
  \label{3d_segm_sfd_barc}
\end{subfigure}

\caption{Examples of topological errors. Two 3D segmentations - (a) - ground truth and (b) - prediction - are indistinguishable by their persistence barcodes (c), (d). At the same time, $\text{SFTD}$ between ground truth and prediction equals to $1$, see (e). A pair of points inside each object in (a), (b) corresponds to the dissimilar topological features, voids, where birth/death events are depicted by red/green dots.}
\label{fig:3d_segm_errors}
\end{figure}

\section{Conclusions}

In this paper, we have proposed new topological tools: F-Cross-Barcode and Scalar Function Topology Divergence (SFTD), which measure the dissimilarity of multi-scale topology between sublevel sets of two functions having a common domain. Functions can be defined on an undirected graph or Euclidean space of any dimensionality.
We have shown that our tools outperform the standard approach: evaluating Wasserstein distance between persistence barcodes, calculated independently for sublevel sets of scalar functions. 
F-Cross-Barcode and SFTD take into account the localization of topological features.
SFTD, as a loss function, is beneficial in computer vision, since it helps to enforce topological similarity at multiple scales. Also, the proposed tools provide useful visualizations of areas of topological dissimilarities.
In the experimental section, we demonstrated that SFTD improved the reconstruction of cellular 3D shapes from 2D fluorescence microscopy images and identified topological errors in 3D segmentation.
Additionally, we showed that SFTD outperforms Betti matching loss in 2D segmentation problems.
We believe that the proposed tools will have more applications in computer vision and learning from graph data.


\par\vfill\par

\section*{Acknowledgements}
The work was supported by the Analytical center under the RF Government (subsidy agreement 000000D730321P5Q0002, Grant No. 70-2021-00145 02.11.2021)

%
%
\bibliographystyle{splncs04}
\bibliography{biblio}

\newpage

%
%
%
%
\appendix
\section{Proof of SFTD basic properties}
\label{sec:stability}

\begin{proposition}\label{prop:stab2}\textbf{(SFTD Stability)}
   For any scalar functions $f$, ${{f}'}$, $g$, 
 ${{g}'}$ on a grid, or on a graph vertex set: 
  \vskip-0.25in
\begin{multline}
    d_B(\text{F-Cross-Barcode} _ k(f,g), \text{F-Cross-Barcode} _ k({f}',{g}')) \leq \\ \max(\max _ {i} \lvert f(i)-{f}'(i)\rvert ,\max _ {i} \lvert g(i)-{g}'(i)\rvert ). \label{eq:dB}
\end{multline}
 \vskip-0.15in   
\end{proposition}
Here $d_B$ denotes the bottleneck distance between persistence barcodes \cite{chazal2017introduction}.
\begin{proof}The proof  is parallel in the simplicial and in the cubical cases. {By construction,  $\text{F-Cross-Barcode}_k(f,g)$, $\text{F-Cross-Barcode}_k(f',g')$ are the} $k-$th persistence barcodes of the lower-star filtrations induced by the functions $\tilde{f}$, $\tilde{f}'$
    on the doubled graph $\tilde{\mathcal{G}}$ from section \ref{sec:graphs}.  If $\max(\max _ {i} \lvert f(i)-{f}'(i)\rvert ,\max _ {i} \lvert g(i)-{g}'(i)\rvert )=\varepsilon$, then $\max_{{j}\in \tilde{\mathcal{V}}} \lvert \tilde{f}(j)-\tilde{f}'(j)\rvert\leq \varepsilon$. 
     Hence the filtration of each simplex in the lower star filtration induced by $\tilde{f}$ on   $\tilde{\mathcal{G}}$ changes at most by $\varepsilon$ under the perturbation $\tilde{f}'$.  It follows from e.g. the description of metamorphoses of canonical forms in \cite{barannikov1994framed} that the birth or the death of each segment in the $k-$th barcode of $\tilde{\mathcal{G}}$ changes under such perturbation at most by $\varepsilon$. 
\end{proof}
\begin{proposition} 
 For any scalar functions $f$, ${f}'$: 
 \vskip-0.1in
 $$\lVert \text{F-Cross-Barcode}_k(f,{f}')\rVert_B \leq \max_{i} \lvert f(i)-{f}'(i)\rvert.$$   
 \vskip-0.1in
\end{proposition}
\begin{proof}
    Apply equation (\ref{eq:dB}) with $g=g'=f'$.
\end{proof}

Let $C^{f\leq\alpha}(\mathcal{G})$ denotes the simplicial (cubical) subcomplex of the simplicial (cubical) complex associated with graph (lattice) $\mathcal{G}$, containing all simplices (cells) on which $f\leq\alpha$. The proof of propositon \ref{prop_std} follows from the exactness of the sequence in equation \ref{eq:longseq} below.
\begin{theorem}\label{thr:basic}{F-Cross-Barcode}$_k(f,g)$ has the folowing properties: 
\begin{itemize}[topsep=0pt,noitemsep,nolistsep, partopsep=0pt, parsep=0ex, leftmargin=*]
\item if $f(i)=g(i)$ for any $i\in \mathcal{V}$, then {F-Cross-Barcode}$_k(f,g)=\varnothing$ for any $k\ge 0$;
\item if $g(i)=\min_{j\in \mathcal{V}}{f(j)}$ for any $i$, then for all $k\ge 0$: {F-Cross-Barcode$_{k+1}(f,g)=\text{Barcode}_{k}(f)$} the standard barcode of the lower star filtration induced by $f$ on graph $\mathcal{G}$;
    \item for any value of threshold $\alpha$, the following sequence of natural linear maps of homology groups
\begin{multline}
    \xrightarrow{r_{3i+3}} H_{i}(C^{f\leq\alpha}(\mathcal{G})) \xrightarrow{r_{3i+2}} H_i(C^{\min(f,g)\leq\alpha}(\mathcal{G}))\xrightarrow{r_{3i+1}}\\ \xrightarrow{r_{3i+1}} H_i(C^{\tilde{f}\leq\alpha}(\tilde{\mathcal{G}})) \xrightarrow{r_{3i}} H_{i-1}(C^{f\leq\alpha}(\mathcal{G}))\xrightarrow{r_{3i-1}}\ldots\\\ldots\xrightarrow{r_1} H_{0}(C^{f\leq\alpha}(\mathcal{G}))\xrightarrow{r_0} H_{0}(C^{\min(f,g)\leq\alpha}(\mathcal{G}))\xrightarrow{r_{-1}}H_0(C^{\tilde{f}\leq\alpha}(\tilde{\mathcal{G}})) \xrightarrow{r_{-2}}0\label{eq:longseq}
\end{multline} is exact, i.e. for any $j$ the kernel of the map $r_{j}$ is the image of the map $r_{j+1}$.
\end{itemize}
\end{theorem}
\begin{proof}
    The proof of the first two properties is immediate from the definiton of F-Cross-Barcode$_{k}(f,g)$. The proof of the third property follows from the proposition \ref{prop:triangl}. 
\end{proof}

\begin{proposition} \label{prop:triangl}
 The embeddings of graphs $\mathcal{G}^{f\leq\alpha}\subseteq \mathcal{G}^{\min(f,g)\leq\alpha}\subset\tilde{\mathcal{G}}^{\tilde{f}\leq\alpha} $
give distinguished triangles, see \cite{gelfand2002methods}, in the homotopy category of chain complexes: \begin{equation}
     C^{{f}\leq\alpha}({\mathcal{G}})\to  C^{\min(f,g)\leq\alpha}(\mathcal{G})\to C^{\tilde{f}\leq\alpha}(\tilde{\mathcal{G}})\to C^{{f}\leq\alpha}({\mathcal{G}})[-1].
     \label{eq:triang1}
\end{equation} 
\begin{proof}
    The proof is parallel to the proof of proposition A.3 from \cite{barannikov2021representation}.
\end{proof}

\end{proposition}
\section{Connection to Representation Topology Divergence}

Representation Topology Divergence (RTD) \cite{barannikov2021representation} is a tool to compare two representations $R_1, R_2$ of some set of objects $\mathcal{V}$. RTD measures the dissimilarity in multi-scale topology
between two point clouds of equal size with a one-to-one correspondence.
Here we establish connections between SFTD and RTD and highlight differences.

A unified view of SFTD and RTD involves two elements:
\begin{enumerate}
    \item calculating a persistence barcode of a filtered simplicial complex derived from the double graph $\tilde{\mathcal{G}}$, see Fig. \ref{fig:doubled_graph};
    \item usage of a specific filtration function $T: \mathcal{C}(\tilde{\mathcal{G}}) \to \mathbb{R}$.
\end{enumerate}


Theoretically, the only restriction for $T$ is that $T(C_1) \le T(C_2)$ when $C_1 \subseteq C_2$.
For SFTD, the filtration function $T$ is defined in vertices by a vertex function $\tilde{f}$ (see Section \ref{sec:method}) and extended to the rest of simplices by a formula $T(C) = \max_{v \in C} \tilde{f}(v)$.
While for RTD, $T(C)$ conforms to the Vietoris-Rips filtration with distance-like weights and is defined on vertices and edges as follows:
\begin{gather*}
    T(A_i) = T(A_j) = T(O) = 0\nonumber \\
    T(A'_i, A'_j)=\min(\text{dist}(R_1(i), R_1(j)), \text{dist}(R_2(i), R_2(j))), \nonumber\\
    T(A_i,A'_j)=T(A_i,A_j)=\text{dist}(R_1(i), R_1(j)), \nonumber\\
    T(A_i,A'_i)=T(O,A_i)=0.\,\, \label{eq:matrO}
\end{gather*}
Here, for two objects $i, j \in \mathcal{V}$ we have their representations $R_1(i), R_2(i)$ and a distance function $\text{dist}(R_1(i), R_2(i))$. For the rest of the simplices, the filtration function is extended by a formula $T(C) = \max_{v_1, v_2 \in C, (v_1, v_2) \in \tilde{\mathcal{E}}} T((v_1, v_2))$.


\section{Computational complexity of SFTD}
\label{app:complexity}

The dominating step in SFTD computation is an evaluation of a persistence barcode.
Generally, it is at worst cubic in the number of cubes/simplexes involved. In practice, the computation is faster since the boundary matrix is typically sparse for real datasets. 
SFTD for $512^2$ or $64^3$ lattices can be calculated in seconds.

\section{Details on the experiment with SHAPR}

For experiments, we used the official repository\footnote{\url{https://github.com/marrlab/SHAPR_torch}}. We closely followed the official pipeline and trained with the loss:
$
\mathcal{L}_{SHAPR} + \lambda \cdot SFTD,$
where $\lambda$ was selected from the set $\{1.0, 0.5, 0.25, 0.125, 0.06\}$ by Dice error on cross-validation.
For baselines (SHAPR, SHAPR+W.D.) we used predictions from a zip-archive from the official repository.

\section{Details on the experiment with 3D segmentation}

For experiments, we used the repository \footnote{\url{https://github.com/Project-MONAI/research-contributions/tree/main/SwinUNETR/BRATS21}} with the implementation of Swin UNETR model pretrained on BraTS21 dataset of the MONAI Project framework \footnote{\url{https://github.com/Project-MONAI}}. The ground truth and predicted segmentations have three channels corresponding to NCR, ED, ET parts of tumor. In the experiments, we analyze and visualize the segmentations corresponding to each channel separately. To visualize the distinctive topological features provided by F-Cross-Barcode, we separately compute SFTD between ground truth and predicted segmentations and vice versa. Topological features provided by SFTD between the ground truth and prediction are depicted in the ground truth segmentation while SFTD between the prediction and ground truth are depicted in the predicted segmentation. In each example, the object's voxels have values of $1$ while the background's voxels have values of $0$. In the visualizations of the persistence barcodes and F-Cross-Barcodes, we omit the infinite half-line in $H_0$ for the ease of perception.


\section{Details on the experiment with graph Laplacian eigenvectors}

We used a slightly different definition of the F-Cross-Barcode, where the $\min(F,G)$ matrix in the Algorithm \ref{alg:sfd_barcodes} was computed slightly differently: 
$\min(F,G)$ was equal to an element-wise minimum of $F, G$, where the matrix $G$ was created in the same way for the function $g$ as the matrix $F$ for the function $f$.
This is somewhat more natural in case of functions defined on general graphs. For sufficiently smooth functions, that is, the values of the functions in nearby vertices differ no more that $\varepsilon$, this procedure gives the F-Cross-Barcode such that its bottleneck distance from the previous defined one doesn't exceed $\varepsilon$.

\section{Additional results for 3D segmentation}

In this section, we provide additional examples of topological errors between the ground truth and predicted 3D segmentations, see Figures \ref{fig:3d_segm_errors_ex1}, \ref{fig:3d_segm_errors_ex2}. The predicted segmentations in figures \ref{3d_segm_pred_ex1}, \ref{3d_segm_pred_ex2} have incorrect localization of clusters compared to the ground truth in figures \ref{3d_segm_label_ex1}, \ref{3d_segm_label_ex2}. In both cases, this difference is indistinguishable by persistence barcodes (see figures \ref{3d_segm_pers_barc_label_ex1}, \ref{3d_segm_pers_barc_pred_ex1}, \ref{3d_segm_pers_barc_label_ex2}, \ref{3d_segm_pers_barc_pred_ex2}) but is revealed by the F-Cross-Barcodes (see figures \ref{3d_segm_sfd_barc_ex1}, \ref{3d_segm_sfd_barc_ex2}).

\begin{figure}[tbp]
\centering
\begin{subfigure}{.32\linewidth}
  \centering
  \includegraphics[width=1\linewidth]{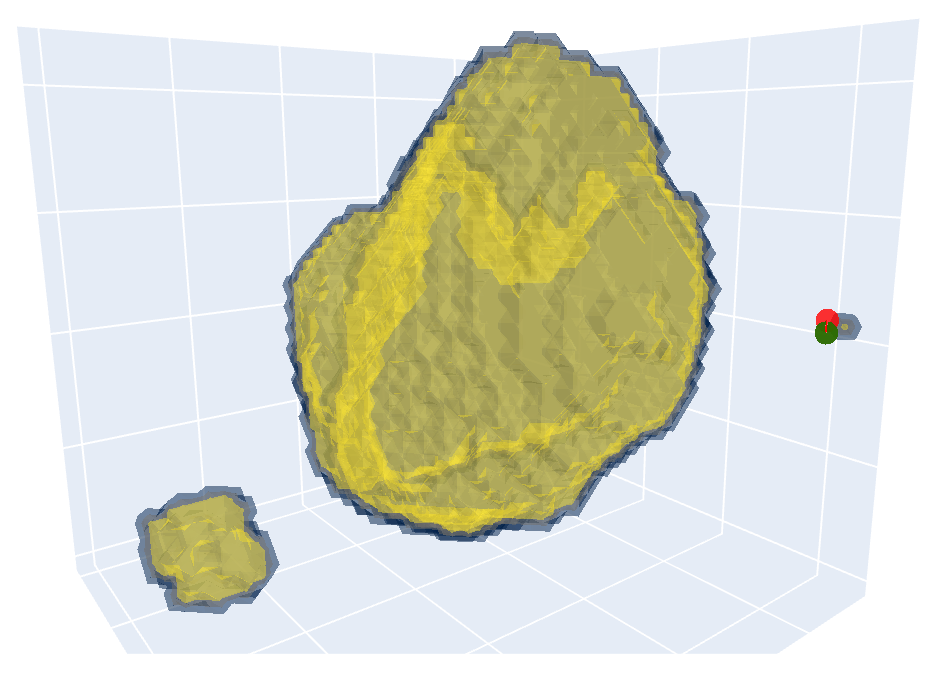}
  \caption{Ground truth}
  \label{3d_segm_label_ex1}
\end{subfigure}%
\hfill
\begin{subfigure}{.32\linewidth}
  \centering
  \includegraphics[width=1\linewidth]{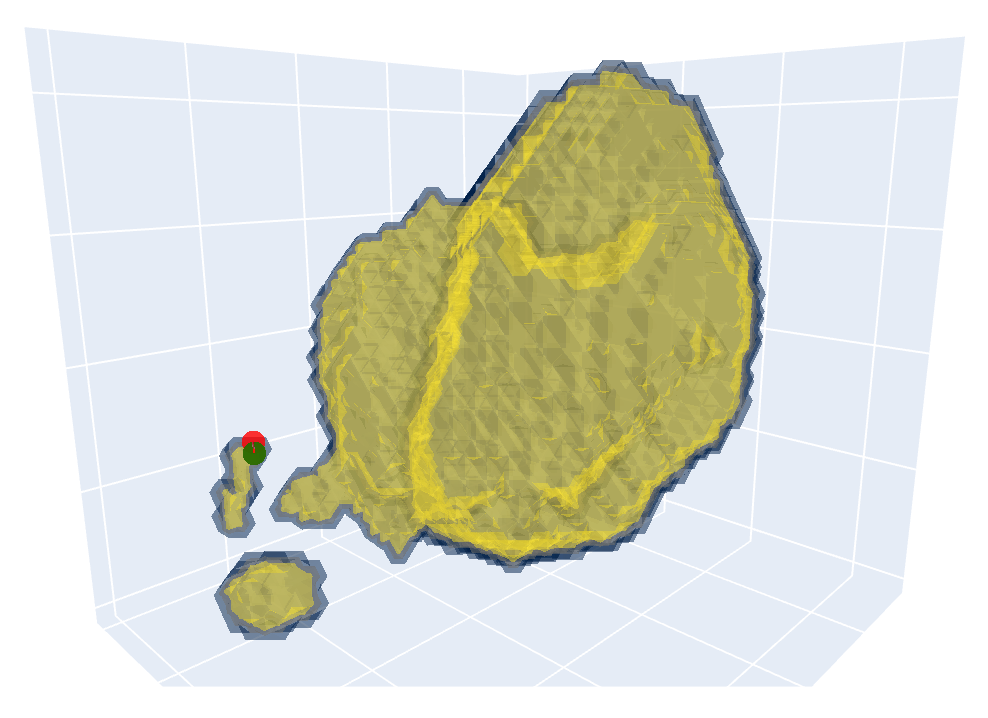}
  \caption{Prediction}
  \label{3d_segm_pred_ex1}
\end{subfigure}
\begin{subfigure}{.32\linewidth}
  \centering
  \includegraphics[width=1\linewidth]{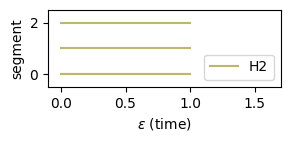}
  \caption{Pers. barc. of ground truth}
  \label{3d_segm_pers_barc_label_ex1}
\end{subfigure}%
\hfill
\begin{subfigure}{.32\linewidth}
  \centering
  \includegraphics[width=1\linewidth]{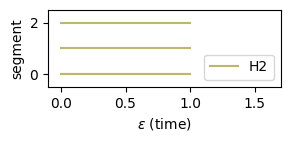}
  \caption{Pers. barc. of prediction}
  \label{3d_segm_pers_barc_pred_ex1}
\end{subfigure}
\hfill
\begin{subfigure}{.32\linewidth}
  \centering
  \includegraphics[width=1\linewidth]{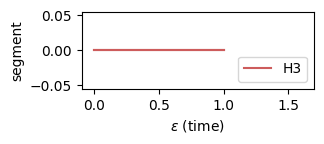}
  \caption{F-Cross-Barcode}
  \label{3d_segm_sfd_barc_ex1}
\end{subfigure}

\caption{Examples of topological errors in 3D segmentation. Predicted segmentation (b) has incorrect localization of one cluster compared to the ground truth (a). This difference is indistinguishable by persistence barcodes (c), (d) but is revealed by F-Cross-Barcode (e).}
\label{fig:3d_segm_errors_ex1}
\end{figure}

\begin{figure}[tbp]
\centering
\begin{subfigure}{.32\linewidth}
  \centering
  \includegraphics[width=1\linewidth]{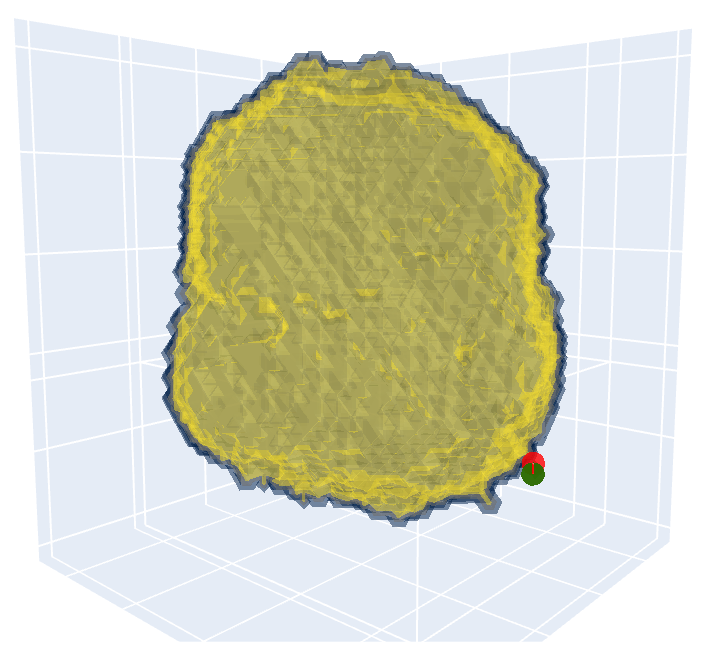}
  \caption{Ground truth}
  \label{3d_segm_label_ex2}
\end{subfigure}%
\hfill
\begin{subfigure}{.32\linewidth}
  \centering
  \includegraphics[width=1\linewidth]{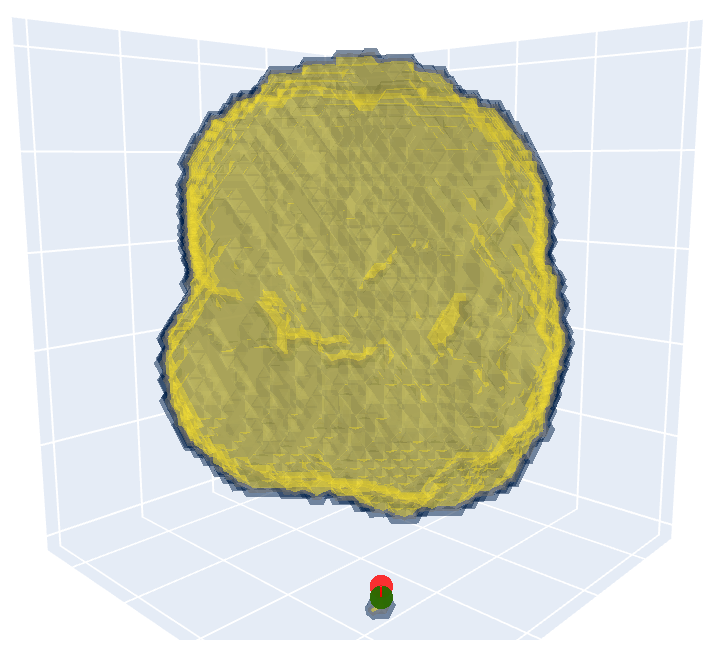}
  \caption{Prediction}
  \label{3d_segm_pred_ex2}
\end{subfigure}
\begin{subfigure}{.32\linewidth}
  \centering
  \includegraphics[width=1\linewidth]{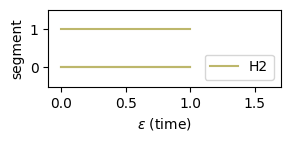}
  \caption{Pers. barc. of ground truth}
  \label{3d_segm_pers_barc_label_ex2}
\end{subfigure}%
\hfill
\begin{subfigure}{.32\linewidth}
  \centering
  \includegraphics[width=1\linewidth]{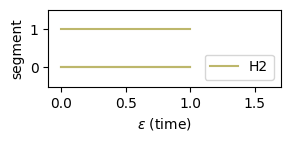}
  \caption{Pers. barc. of prediction}
  \label{3d_segm_pers_barc_pred_ex2}
\end{subfigure}
\hfill
\begin{subfigure}{.32\linewidth}
  \centering
  \includegraphics[width=1\linewidth]{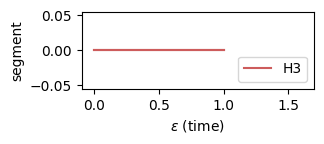}
  \caption{F-Cross-Barcode}
  \label{3d_segm_sfd_barc_ex2}
\end{subfigure}

\caption{Examples of topological errors in 3D segmentation. The ground truth segmentation (a) has a tiny cluster that is incorrectly predicted by the model (b). This difference is indistinguishable by persistence barcodes (c), (d) but is revealed by F-Cross-Barcode (e).}
\label{fig:3d_segm_errors_ex2}
\end{figure}

\section{Additional experiments for 2D segmentation}
\label{app:2d_segmentation}

In 2D segmentation, optimizing loss functions based on pixel-overlap is often insufficient to capture the correct topology of segmented objects. Betti matching \cite{stucki2023topologically} is a topological metric and loss function for supervised image segmentation. When used as a loss function, Betti matching improves topological performance of segmentation models while preserving the volumetric quality.

We strictly followed the experimental setup from \cite{stucki2023topologically} to compare SFTD with Betti matching loss on CREMI \cite{funke2018large} and colon cancer cell (Colon) \cite{carpenter2006cellprofiler} datasets. We replaced the topological loss term in Betti matching with SFTD with $p=2$ in up to $2$ dimensions. Table \ref{tbl:2d_segm} provides experimental results. Our method has similar segmentation quality (Dice, clDice, accuracy) and lower topological errors ($\tau^{\text{err}}$, $\beta^{\text{err}}$) - up to 35\%. Also, evaluation of SFTD loss is $\sim 35$ times faster than Betti matching loss.

\begin{table*}[hp!]
\caption{Errors of 2D segmentation for CREMI and Colon datasets.}
\label{tbl:2d_segm}
\begin{center}
\begin{small}
\begin{sc}
\begin{tabular}{lccccc}
\toprule
Method & $\text{Dice}$\,$\uparrow$ & $\text{clDice}$\,$\uparrow$ & $\text{Acc.}$\,$\uparrow$ & $\tau^{\text{err}}$\,$\downarrow$ & $\beta^{\text{err}}$\,$\downarrow$\\
\midrule
\multicolumn{6}{c}{CREMI dataset \cite{funke2018large}} \\ \midrule
Betti matching \cite{stucki2023topologically}            & \textbf{0.893} & \textbf{0.941} & \textbf{0.959} & $129.80$ & $79.16$ \\
SFTD & $0.885$ & \textbf{0.941} & $0.955$ & \textbf{126.48} & \textbf{62.64} \\ \midrule
\multicolumn{6}{c}{Colon dataset \cite{carpenter2006cellprofiler}} \\ \midrule
Betti matching \cite{stucki2023topologically} & \textbf{0.907} & $0.871$ & \textbf{0.975} & $32.00$ & $21.50$ \\
SFTD & $0.901$ & \textbf{0.889} & $0.971$ & \textbf{21.50} & \textbf{14.00} \\
\bottomrule
\end{tabular}
\end{sc}
\end{small}
\end{center}
\end{table*}

\end{document}